\documentclass[sigconf]{acmart}
\usepackage{multicol,url,color}
\usepackage{multirow}
\usepackage{rotating,booktabs}
\usepackage{adjustbox,subcaption}
\usepackage{amsmath,amsfonts}
\usepackage{algorithm}
\usepackage[noend]{algpseudocode}
\usepackage{graphicx}
\usepackage[multiple]{footmisc}
\usepackage{caption}
\usepackage{bbding}
\usepackage{colortbl}
\usepackage{hhline}
\usepackage{multirow}
\usepackage{comment}
\usepackage{float}
\usepackage{textcomp}
\usepackage{pifont}
\makeatother
\usepackage{float}
\usepackage{fancyhdr}
\usepackage{mathtools} 
\usepackage{graphics}
\usepackage{placeins}
\usepackage{array}
\usepackage{enumitem}
 \usepackage{color}

\usepackage{pifont}

\theoremstyle{observation}
 \newtheorem{observation}{Observation}[subsection]
 
 \usepackage{tikz}
  \makeatletter

\newcommand{\vect}[1]{\mathbf{#1}}
\newcommand{\Wskip}[1]{ }

\AtBeginDocument{%
  \providecommand\BibTeX{{%
    \normalfont B\kern-0.5em{\scshape i\kern-0.25em b}\kern-0.8em\TeX}}}

\setcopyright{acmcopyright}
\copyrightyear{2021}
\acmYear{2021}
\acmDOI{10.1145/1122445.1122456}

\begin{document}

%\title{Research: Towards the Complete Sparsification of Graph Convolutional Networks}
\title{Triple Sparsification of Graph Convolutional Networks without Sacrificing the Accuracy}

\author{Md. Khaledur Rahman}
\affiliation{%
  \institution{Indiana University Bloomington}
  \country{}
}
\email{morahma@iu.edu}

\author{Ariful Azad}
\affiliation{%
  \institution{Indiana University Bloomington}
  \country{}
 }
 \email{azad@iu.edu}

%\renewcommand{\shortauthors}{Rahman and Azad}

\begin{abstract}
  Graph Neural Networks (GNNs) are widely used to perform different machine learning tasks on graphs. As the size of the graphs grows, and the GNNs get deeper, training and inference time become costly in addition to the memory requirement. Thus, without sacrificing accuracy, graph sparsification, or model compression becomes a viable approach for graph learning tasks. A few existing techniques only study the sparsification of graphs and GNN models. In this paper, we develop a SparseGCN pipeline to study all possible sparsification in GNN. We provide a theoretical analysis and empirically show that it can add up to 11.6\% additional sparsity to the embedding matrix without sacrificing the accuracy of the commonly used benchmark graph datasets. %Our empirical results show that the model alone could also be compressed up to $10.6\times$. %Our combined sparsity results outperform state-of-the-art sparsity results. 
\end{abstract}




\maketitle
\section{Introduction}
% introduce GCN and history
In recent years, Graph Neural Networks (GNNs) have become popular for various learning tasks \cite{wu2020comprehensive,zhang2020deep} such as node classification \cite{kipf2016semi} and link prediction \cite{zhang2018link}.
%Over the years, a diverse collection of GNN methods have been discussed in the literature including  Graph Convolutional Network (GCN)~\cite{kipf2016semi}, GraphSAGE~\cite{hamilton2017inductive}, and FastGCN~\cite{chen2018fastgcn} to name a few.
%introduces a faster training method by sampling a batch of vertices, and Cluster-GCN
Popular message-passing-based GNNs such as Graph Convolutional Network (GCN)~\cite{kipf2016semi} and GraphSAGE~\cite{hamilton2017inductive} learn the representation of a node by aggregating information from its neighbors.
Dominant computations in such a GNN depend on the input graph, the hidden layer representation of nodes (embeddings), and weight matrices.
Consequently, the computational complexity of GNN training and inference depends completely on the number of non-zero entries in these matrices.
As graphs become bigger and GNNs become deeper, GNNs are increasingly demanding more computational resources. 
This paper aims to address this challenge by exploiting sparsity in all aspects of GNNs. 

It is well known that GCN training and inference are computationally intensive for large-scale graphs. For example, You et al.~\cite{you2021gebt} explained that a 2-layer GCN model with 32-dimensional embeddings in hidden layers may require 19 GFLOPs (FLOPs: floating point operations) on the Reddit graph with about 232K nodes and and 114M edges. 
In comparison, the popular ResNET50 model~\cite{he2016deep} requires 8 GFLOPs for a pass over ImageNet~\cite{russakovsky2015imagenet,deng2009imagenet}.
Furthermore, the actual performance of GCN is significantly worse in practice because GCN requires sparse matrix multiplications that do not attain the peak performance of processors and GPUs~\cite{gale2020sparse}.
As a result, it is extremely difficult to train a GCN with a graph with billions of nodes even with hundreds of GPUs.   

A viable solution to GCN's enormous computation demand is to sparsify input graphs, intermediate representations, and model parameters.
%However, such sparsification is feasible only when 
%there are redundancies within input graphs and GCN models so that 
%the sparse solution does not reduce the accuracy of GNNs.
Several prior work~\cite{chen2021unified, you2021gebt, li2020sgcn} have already showed that limited sparsification of the input graph and/or model weight do not reduce the accuracy of GCNs. 
Following on the footsteps of these pioneering work, we demonstrate that all aspects of data, features, embedding, and GCN weight matrices can be sparsified without sacrificing the performance of GCNs. 

% now discuss what type of matrices are involved and introduce triple sparsification 

% content of SparseGCN
In each layer of a GCN (and most other GNNs), the hidden representation or embedding in a layer is determined by two matrix multiplications involving the adjacency matrix, the embedding matrix of the previous layer (input features being the starting embedding), and the weight matrix \cite{kipf2016semi}. While the adjacency matrix almost always comes in sparse formats, the other matrices are usually dense in most GNNs. Given the involvement of three matrices in the computation of each layer, there are opportunities to sparsify all three matrices to bring down the computational requirements. 
However, all previous work on sparse GNNs considered sparsifying the input graph and/or model weights~\cite{chen2021unified, you2021gebt, li2020sgcn}. 
No previous work considered sparsifying the embedding matrices even though they often contain more non-zero elements than the adjacency and weight matrices. 
Thus, the sparsification of the embedding matrix is expected to bring down the computational requirement of GNNs significantly. 
%All previous work on Sparse GNNs considered sparsifying the input graph and/or model weights. 
In this paper, we consider sparsifying  all three matrices involved in a GNN layer. We theoretically demonstrate that embedding sparsification accompanied by graph and weight sparsifications reduces the number of Multiply-ACcumulate (MAC) operations needed to train a GNN model. The reduced computational requirement could be achieved without sacrificing the test accuracy. 
As with previous work, we demonstrate the impact of sparsity by counting the number of MAC operations needed for GNN training and inference. Our goal is not to show the actual reduction of runtime that depends on the implementation of sparse matrix multiplications and hardware platforms~\cite{gale2020sparse}.

%and enable on-the-fly regularization to the model. 
% 
%Thus, in this paper, we fill this gap by studying a completely sparse GCN. More specifically, we focus on sparsifying the input graph, embedding, and the model weight without sacrificing accuracy. 
%We also explore a sensitivity-based sparsification for the model weight that has not been studied by other existing GNN papers. 
The main contributions of this paper are summarized below.
\begin{itemize}
    \item We study the impact of complete sparsification of all matrices (the input graphs, network weight matrices, and hidden layer representations) in the GCN model.
    % To the best of our knowledge, this is the first study of embedding sparsification in GNNs. % the pruning of hidden representation with several sparsification techniques along with the naive random approach. We analyze the combined effect of introducing sparsity to all matrices of GCN.
    \item We use several sparsification techniques including Top-$k$, Sorting, and Sensitivity-based methods. %introduce the effectiveness of Top-$k$ and Sensitivity-based sparsifications for GCN. We also study the previously used random and Sorting-based sparsifications.
    % We demonstrate that different sparsification techniques work the best for different matrices.  
    We derive theoretical bounds on the reduction of MAC operations for different induced sparsity.
    \item Our experimental results shows that SparseGCN can attain a higher percentage of overall sparsity than the existing methods without sacrificing the baseline accuracy. 
\end{itemize}
\begin{figure*}[!ht]
    \centering
    \includegraphics[width=0.85\linewidth]{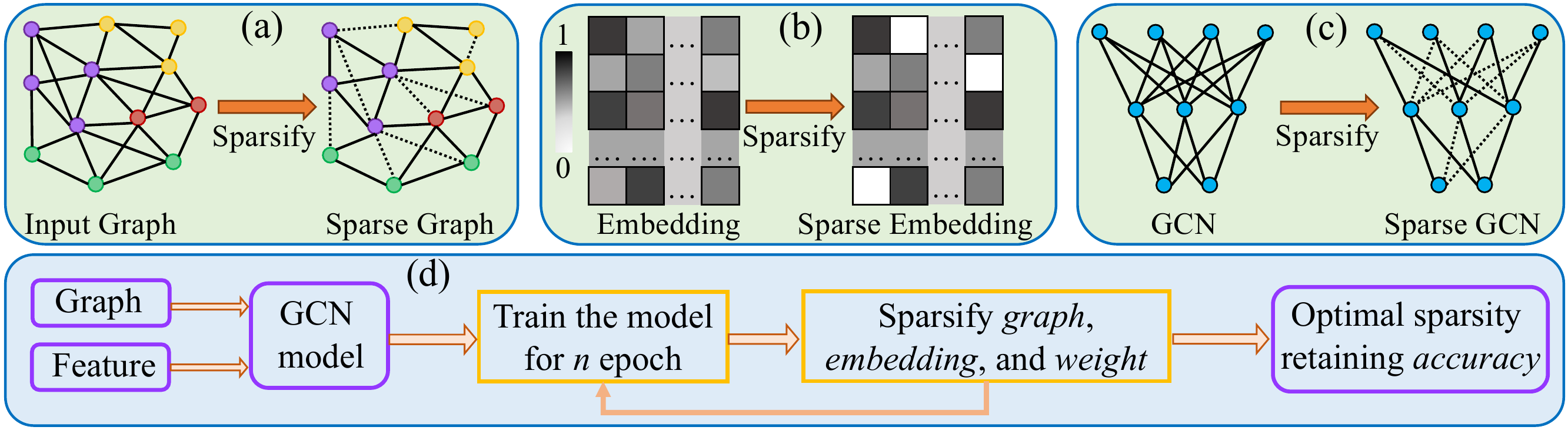}
    %\vspace{-0.45cm}
    \caption{Example of (a) graph sparsification by removing edges (dotted lines), (b) embedding matrix sparsification by assigning zero, and (c) weight sparsification by removing edges in the GCN model. (d) A workflow of our sparse training procedure. An input graph with node features is fed to the GCN model. It is trained for $n$ epochs and records the test accuracy. After that, a sparse routine introduces sparsity in adjacency, embedding, and weight matrice, and then the model is trained again. The sparsification and re-training steps continue until we achieve an optimal level of sparsity without sacrificing test accuracy.}
    %\vspace{-0.4cm}
    \label{fig:workflow}
\end{figure*}
%\vspace{-0.35cm}
\section{Background}
%\vspace{-0.15cm}
\subsection{Notations}
Let $G(V, E)$ be a directed graph, where $V=\{v_1, \ldots v_n\}$ is the set of vertices and $E=\{e_1, \ldots, e_m\}$ is the set of edges such that $|V|=n$, and $|E|=m$. 
%Here, each $e_i = (v_a, v_b) \in E$ means that there is an edge between $v_a$ and $v_b$. 
Let $\vect{X}\in \mathbb{R}^{n\times d}$ store $d$-dimensional input features of $n$ vertices. 
%We represent the adjacency matrix of $G$ by 
$\vect{A}$ denotes the adjacency matrix of $G$ where $\vect{A}_{ij}=1$ if $(v_i, v_j)\in E$, otherwise $\vect{A}_{ij}=0$. The core computation at the $l$-th layer of GCN \cite{kipf2016semi} is:
%\vspace{-0.25cm}
\begin{equation}
\label{eqn:gcn}
    \vect{H}^{l} = \sigma(\vect{\hat{A}}\vect{H}^{l-1}\vect{W}^{l-1})
\end{equation}
where, $\vect{W}$ is the neural network weight matrix, $\sigma$ is the activation function and $\vect{H}$ is the hidden representation of corresponding convolutional layer such that $\vect{H}^0 = \vect{X}$. In Eqn. \ref{eqn:gcn}, $\vect{\hat{A}}$ is the renormalized graph Laplacian matrix such that $\vect{\hat{A}}=\vect{\tilde{D}}^{-\frac{1}{2}}\vect{\tilde{A}}\vect{\tilde{D}}^{\frac{1}{2}}$, $\vect{\tilde{A}} = \vect{I} + \vect{A}$, where $\vect{I}$ is the identity matrix, and $\tilde{D}_{ii} = \sum_{j}\tilde{A}_{ij}$. We optimize the GCN model using the negative log-likelihood loss function.

\subsection{Related Work}

\textbf{Graph Neural Networks.} 
%In recent years, Graph Neural Networks (GNN) have become popular to solve graph learning tasks such as node classification, link prediction, and graph classification.
Over the last decade, hundreds of GNN methods, libraries, and software have been developed for semi-supervised, unsupervised, and self-supervised learning on graphs~\cite{wu2020comprehensive, liu2022graph,rahman2021comprehensive,rahman2020force2vec,rahman2022markovgnn,velivckovic2017graph,xu2018powerful}. 
Their success in graph learning tasks and their limitations such as over-smoothing and neighborhood explosion problems are also well-documented in the literature~\cite{alon2020bottleneck,oono2019graph,chiang2019cluster}. 
%They primarily focus on learning representation using the information of neighbors in a semi-supervised way~\cite{kipf2016semi,hamilton2017inductive,velivckovic2017graph,xu2018powerful}. 
%However, they are limited by some bottlenecks such as over-smoothing problem or neighborhood explosion problem \cite{alon2020bottleneck,oono2019graph,chiang2019cluster}. 
Here, our focus is the computational challenges arising in GNN training due to the size of different matrices in Eqn. \ref{eqn:gcn}. 
%of large graphs is still a challenging task due to the size of different matrices in Eqn. \ref{eqn:gcn}. 
%Thus researchers focus on optimized resource usages with improved learning strategy in GNN.

% \textbf{Sparsifications of Neural Networks.} The main focus of the sparsity of Convolutional Neural Networks (CNN) is on weight parameters. It helps compress the learned model that results in a faster inference time \cite{han2015learning,han2015deep,frankle2018lottery,liu2021selfish,han2016eie}. For example, Forrest et al. studied the sparsity problem and achieved a compression factor of 50 \cite{iandola2016squeezenet}. Hao et al. studied the pruning of convolutional filters and feature maps which also improves the inference time \cite{li2016pruning}. Siddhant et al. proposed Top-KAST to train a sparse model~\cite{jayakumar2021top}.  There are several survey papers for deep convolutional neural networks that summarize the sparsification or acceleration strategies \cite{hoefler2021sparsity,gale2019state,abadal2021computing}. 

\begin{table}[!htb]
\centering
%\vspace{-0.3cm}
\caption{Papers discussing pruning strategy of a GNN.}
%\vspace{-0.35cm}
\arrayrulecolor{black}
\label{tab:papers}
\begin{tabular}{c|p{1.2cm}p{1.3cm}p{1.2cm}} 
\hline
Paper(s) & Prune $\vect{\hat{A}}$? & Prune $\vect{W}$? & Prune $\vect{H}$? \\ \hline
SGCN~\cite{li2020sgcn}, FastGCN~\cite{chen2018fastgcn}    &        \textbf{\ding{52}}                                                  &                               \textbf{\ding{55}}                              &     \textbf{\ding{55}}                                                             \\ 

ULTH~\cite{chen2021unified}, GEBT~\cite{you2021gebt}     &\textbf{\ding{52}}                                                          &                               \textbf{\ding{52}}                            &                               \textbf{\ding{55}}                                  \\ 

SparseGCN (Ours)     & \textbf{\ding{52}}  & \textbf{\ding{52}}  &  \textbf{\ding{52}}                                                                                                                                          \\
\hline
\end{tabular}
\arrayrulecolor{black}
%\vspace{-0.3cm}
\end{table}
%\vspace{-0.15cm}

\textbf{Sparsifications in GNNs.} 
%Dense hardware can speed up the training~\cite{balog2019fast} of GNNs but as the size of the graph grows, the limitations of memory consumption and computation time become obvious. 
In GNNs, three matrices are involved (see Eqn. \ref{eqn:gcn}) that can be sparsified jointly or separately. In the literature, the sparsification of GNNs has been tackled by researchers considering two directions: (i) input graph or adjacency matrix sparsification, and (ii) model weight sparsification. The former has been studied much in the GNN literature~\cite{chen2018fastgcn,zeng2019graphsaint}; however, the latter has not been explored that much. There are a handful of studies in the literature that discuss such sparsifications of graphs \cite{ye2021sparse,li2020sgcn}. We observe a few studies in the literature for GNN model compression. The quantization or binarization technique can compress the GNN model too~\cite{tailor2020degree,bahri2021binary,wang2021bi}. In recent years, Unified lottery Ticket Hypothesis (ULTH) \cite{chen2021unified} and early-bird ticket \cite{you2021gebt} are the two noticeable works for both graph and model weight sparsifications in GNNs. However, the embedding matrix sparsification has not been explored previously except random \emph{dropout} for regularization. We summarize the pruning or sparsification related papers of GNN in Table \ref{tab:papers}. In this paper, we study the fully sparse GCN by analyzing the sensitivity of different sparse matrices in Eqn. \ref{eqn:gcn}.
%\vspace{-0.25cm}
\section{Methods}
We use GCN as a representative model to induce triple sparsity in three different matrices and call this approach SparseGCN. %We describe different steps of SparseGCN as follows.
%\vspace{-0.25cm}
\subsection{Algorithmic Workflow}
Our training strategy, after adding sparsity to the model, is similar to most of the methods in the literature \cite{chen2021unified,li2020sgcn}. We followed the protocol  used in the unified lottery ticket hypothesis paper~\cite{chen2021unified}. We assume that $m^l_w$ and $m^l_h$ are two differential binary masks of weights and embedding matrices, respectively, on the $l$-th layer. %If we reset (assign 0's) some non-zero entries of $m^l_w$, then, $m^l_w\odot \vect{W}^l$ would necessarily sparsify the weight matrix in the $l$-th layer, where $\odot$ is an element-wise product operator. Similarly, $m^l_h\odot \vect{H}^l$ would sparsify the embedding matrix in the $l$-th layer. 
% We show the schematic diagrams for sparsity in the graph, embedding, and weight in Figs. \ref{fig:workflow}(a)-(c). We also show the workflow in Fig. \ref{fig:workflow}(d). 

\textbf{Graph Sparsification.} 
To add sparsity to the graph or adjacency matrix, we remove a subset of edges from the graph. In Fig. \ref{fig:workflow}(a), we show the removed edges by dotted lines. Since the sparsified graph is used in all layers, we do not create additional masks for different layers. In this step, we remove $a\%$ edges from the weighted adjacency matrix $\vect{\hat{A}}$ using a chosen sparsification technique.

\textbf{Embedding Sparsification.} Unlike previous work, we also sparsify the embedding matrix shown in the schematic diagram \ref{fig:workflow}(b). To sparsify it, we select $h\%$ non-zero entries and set them to 0. We perform this sparsification step with a  Boolean mask $m^l_h$ and then masking entries from the embedding matrix using $m^l_h\odot \vect{H}^l$. %This operation also ensures that masked values in $\vect{H}^l$ are not updated in the current gradient descent iteration. %  by the gradient update procedure.

\textbf{Weight Sparsification.} Similar to previous studies \cite{chen2021unified,you2021gebt}, we also sparsify model weights of GCN in different layers as shown in Fig. \ref{fig:workflow}(c). We mask out a fraction of non-zero entries from the weight matrix using the operation $m^l_w\odot \vect{W}^l$. %Similar to the embedding sparsification, this operation ensures that masked values in $\vect{W}^l$ are not updated in the current iteration.

In the workflow, as shown in Fig. \ref{fig:workflow}(d), we train the GCN model for the input graph with associated node features to obtain the baseline accuracy for the node classification problem. After that, we sparsify the graph, embedding matrix, and weight matrix using a sparsification technique described in Section \ref{sec:sparsification}. Then, we re-train the model with sparsified matrices and compare the accuracy level. If it is similar to the baseline accuracy, we sparsify the sparse matrices again and re-train the model. We continue this process until the accuracy level drops compared to the baseline. Finally, we report the highest sparsity level for which the sparsed model can retain the baseline accuracy.
%\vspace{-7pt}
\subsection{Sparsification Techniques}
\label{sec:sparsification}
For our experiments analyses, we used three pruning strategies: \textbf{(i) Random-} a naive and straightforward approach for the sparsification of any matrices in Eqn. \ref{eqn:gcn}. We pick $p\%$ entries from the non-zero elements of the matrices and mask them out by assigning zero to them. \textbf{(ii) Sorting-based-} a global sparsification technique where we sort all the non-zero entries of the matrices based on the absolute values. Then, we pick the smaller $p\%$ entries and mask them out by assigning zero. The rationale for this technique is that we can remove more insignificant entries globally which will have less computational effect. \textbf{(iii) Top-$k$-} a local sparsification technique where each row of the matrix is sorted based on the absolute value of the non-zero entries. Then, we select the smaller $k$ entries from each row and mask them out by assigning zero. The intuition behind this approach is that the absolute smaller non-zero entries contribute less to the matrix multiplication. 
%\vspace{-7pt}
\subsection{Theoretical Analysis} We focus on the computation in the $l$th layer of a GNN, but the obtained bounds are extendable to other layers. 
We assume that the graph is stored in the compressed sparse row (CSR) format so that the memory requirement is $O(m)$, and the number of Floating-point Operations (FLOPs) in sparse-dense matrix multiplication meets its lower bound~\cite{koanantakool2016communication}. Suppose, the dimensions of $\vect{H}^l$, and $\vect{W}^l$ matrices in the $l$-th convolutional layer of GCN are $n\times d$, and $d\times f$, respectively. Then, we can deduce the following bounds.
%\vspace{-5pt}
\begin{lemma}
\label{lem:upper} The total number of MAC operations on the $l$-th layer is $(fm+dfn-fn)$, where $d > f$.
\end{lemma}
%\vspace{-0.65cm}
\begin{proof}
For $d > f$, the right ordered multiplications (i.e., $\vect{\hat{A}}(\vect{H}^l\vect{W}^l)$) in Eqn. \ref{eqn:gcn} would cost less computations than the left order (i.e., $(\vect{\hat{A}}\vect{H}^l)\vect{W}^l$). The total number of FLOPs for $\vect{T} = \vect{H}^l\vect{W}^l$ are $2dfn -fn$, where $\vect{T}\in \mathbb{R}^{n\times f}$ is a temporary matrix. The number of FLOPs for $\vect{\hat{A}}\vect{T}$ are $fm + f(m-n)$, or $2fm - fn$. Plugging $MACs = \frac{FLOPs}{2}$ (same as ULTH for GCN~\cite{chen2021unified}), the total number of MAC operations on $l$-th layer are $\frac{2dfn - fn + 2fm -fn}{2}$, or $(fm+dfn-fn)$.
\end{proof}
%\vspace{-0.35cm}
\begin{lemma}
\label{lem:lower}
If we introduce $a\%$, $h\%$, and $w\%$ sparsity to $\vect{\hat{A}}$, $\vect{H}^l$, and $\vect{W}^l$, respectively, then the total MAC operations are bounded by $(1-w)(1-h')(fm+dfn-fn)$, where $d > f$ and $a\approx h=h'$.
\end{lemma}
%\vspace{-0.4cm}
\begin{proof}
For simplicity of the computations, we distribute the percentage of sparsity across matrix dimensions. Then, an $a\%$ sparsity of graph would leave $(1-a)m$ non-zero elements in $\vect{\hat{A}}$, an $h\%$ sparsity of embedding would leave $n\times (1-h)d$ non-zero elements in $\vect{H}$, and a $w\%$ sparsity of weight matrix would leave $(1-h)d\times (1-w)f$ non-zero elements in $\vect{W}$. The right to left order multiplication would cost $(1-h)(1-w)dfn + (1-w)(d-dh-1)fn$ FLOPs for $\vect{T} = \vect{H}^l\vect{W}^l$, where $\vect{T}\in \mathbb{R}^{n\times (1-w)f}$ is a temporary matrix. Similarly, the number of FLOPs for $\vect{\hat{A}}\vect{T}$ are $(1-a)(1-w)fm + (1-a)(1-w)(m-n)f$. Then, the total number of FLOPs on the $l$-th layer are as follows:
\begin{flalign*}
\begin{split}
    & = 2(1-h)(1-w)dfn - (1-w)fn + 2(1-a)(1-w)fm\\ &-(1-a)(1-w)fn\\
   %& = (1-w)\{(1-r)dfn + (1-a)fn^2 + (d-rd-1)fn \\ & + (1-a)(n-1)fn\} \\ 
   & \le 2(1-h)(1-w)dfn + 2(1-a)(1-w)fm -2(1-a)(1-w)fn \\
  % & = 2(1-w)\{(1-h)dfn+(1-a)fm-(1-a)fn\} \\
   & = 2(1-w)(1-h')(fm+dfn-fn)
   \end{split}
\end{flalign*}
From the above computational bound, we can derive that the MACs on the $l$-th layer are bounded by $(1-w)(1-h')(fm+dfn-fn)$.
\end{proof}
%\vspace{-0.35cm}
\begin{theorem}
\label{theorem:ratio}
For $a\%$, $h\%$, and $w\%$ sparsity to $\vect{\hat{A}}$, $\vect{H}^l$, and $\vect{W}^l$ matrices, the total reduction factor of MAC operations in GCN is bounded by $\frac{1}{(1-w)(1-h')}$, where $a\approx h=h'$.
\end{theorem}
%\vspace{-0.4cm}
\begin{proof}
We can infer it by taking the ratio of the result of Lemma \ref{lem:upper} to Lemma \ref{lem:lower} that the reduction is bounded by $\frac{1}{(1-w)(1-h')}$.
\end{proof}
%\vspace{-0.2cm}
\textbf{Complexity.} From Lemma  \ref{lem:lower}, we can deduce that the asymptotic time complexity of the sparse training procedure for a $L$-layered GCN would be $O((1-w)(1-h')(m + fn)fL)$, where $f\approx d$, and $|E|=m$. Similarly, the asymptotic memory complexity would be $O(Lf((1-h')n + (1-w)f) + (1-h')m)$.
\section{Experiments}
\subsection{Experimental Setup}
\textbf{Overview.} The goal of our experiments is to explore the sparsity of different matrices in Eqn. \ref{eqn:gcn} without sacrificing the baseline accuracy. We primarily aim to address the following two key Research Questions (RQs): \textbf{RQ1:} How does the performance of node classification vary with individual sparsity techniques? \textbf{RQ2:} What would be the achievable combined sparsity for benchmark graphs?
\begin{table}[!htb]
\centering
%\vspace{-0.3cm}
\caption{Summary of the Graph Datasets.}
%\vspace{-0.30cm}
\label{tab:dataset}
\begin{tabular}{p{1.25cm}|p{0.75cm}|p{0.82cm}|p{0.86cm}|p{1.1cm}|p{1.4cm}} 
\hline
\textbf{Graphs}    & \textbf{Nodes} & \textbf{Edges}   & \textbf{Classes} & \textbf{Features} & \textbf{Avg. Deg.}  \\ 
\arrayrulecolor{black}\cline{1-1}\arrayrulecolor{black}\cline{2-6}
Cora      & 2,708 & 10,556  & 7       & 1,433    & 3.89       \\
Citeseer  & 3,327 & 9,104   & 6       & 3,703    & 2.74       \\ 
Pubmed & 19,717  &   88,648   & 3 &   500 &  4.5 \\

\hline
\end{tabular}
%\vspace{-0.35cm}
\end{table}
\begin{figure*}[!htb]
    \centering
    \includegraphics[width=0.31\linewidth,height=3.1cm]{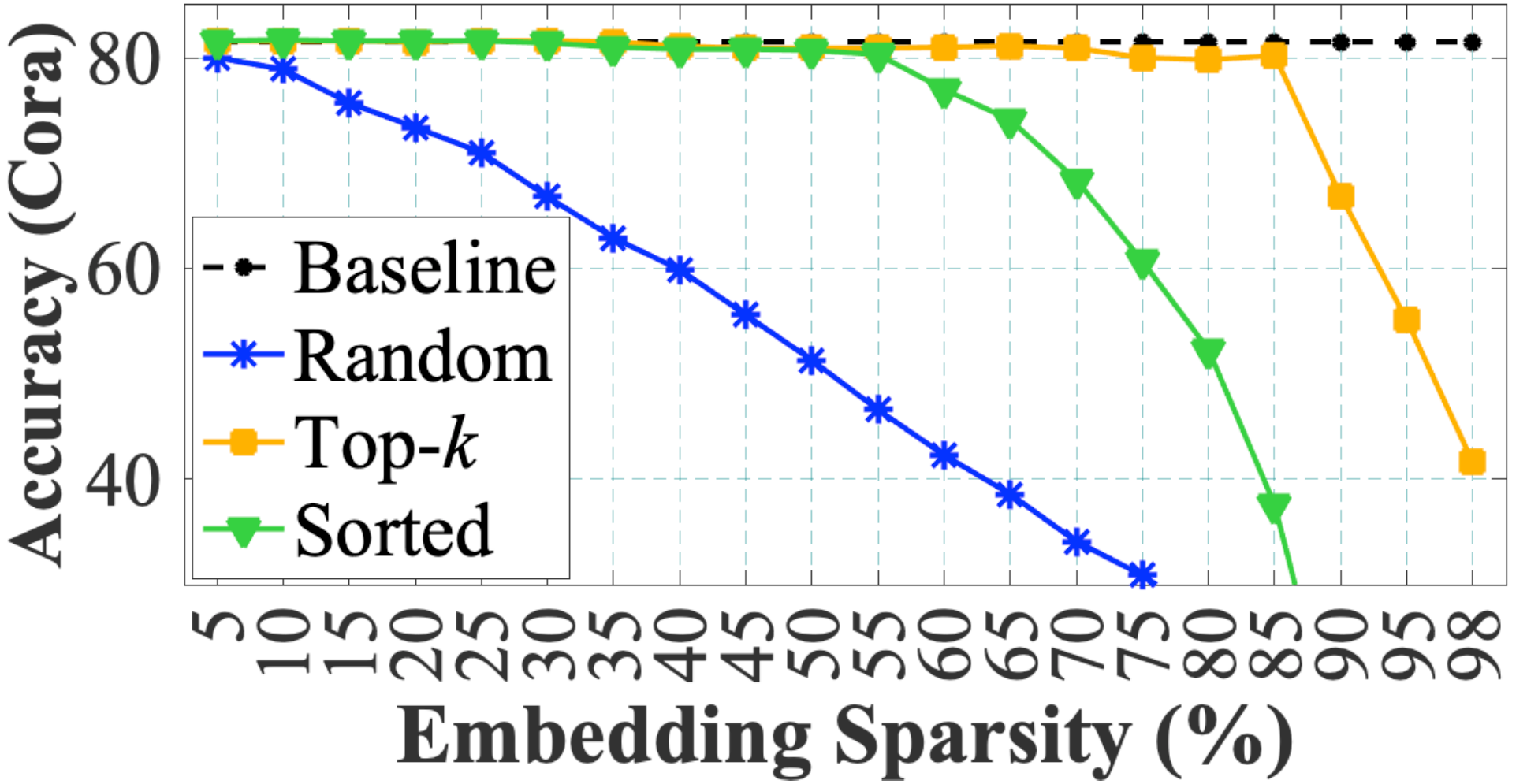}
    \includegraphics[width=0.31\linewidth,height=3.1cm]{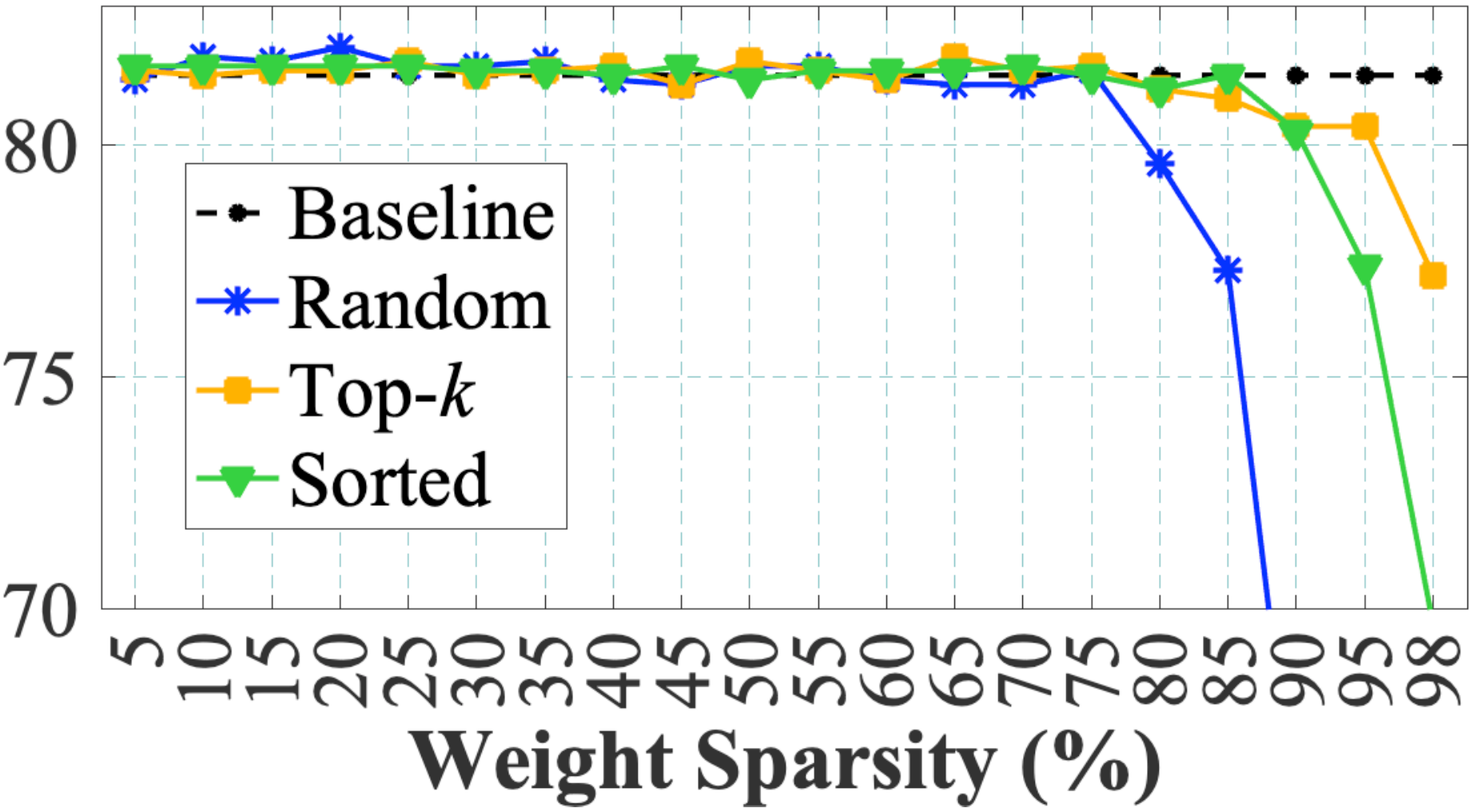}
    \includegraphics[width=0.31\linewidth,height=3.1cm]{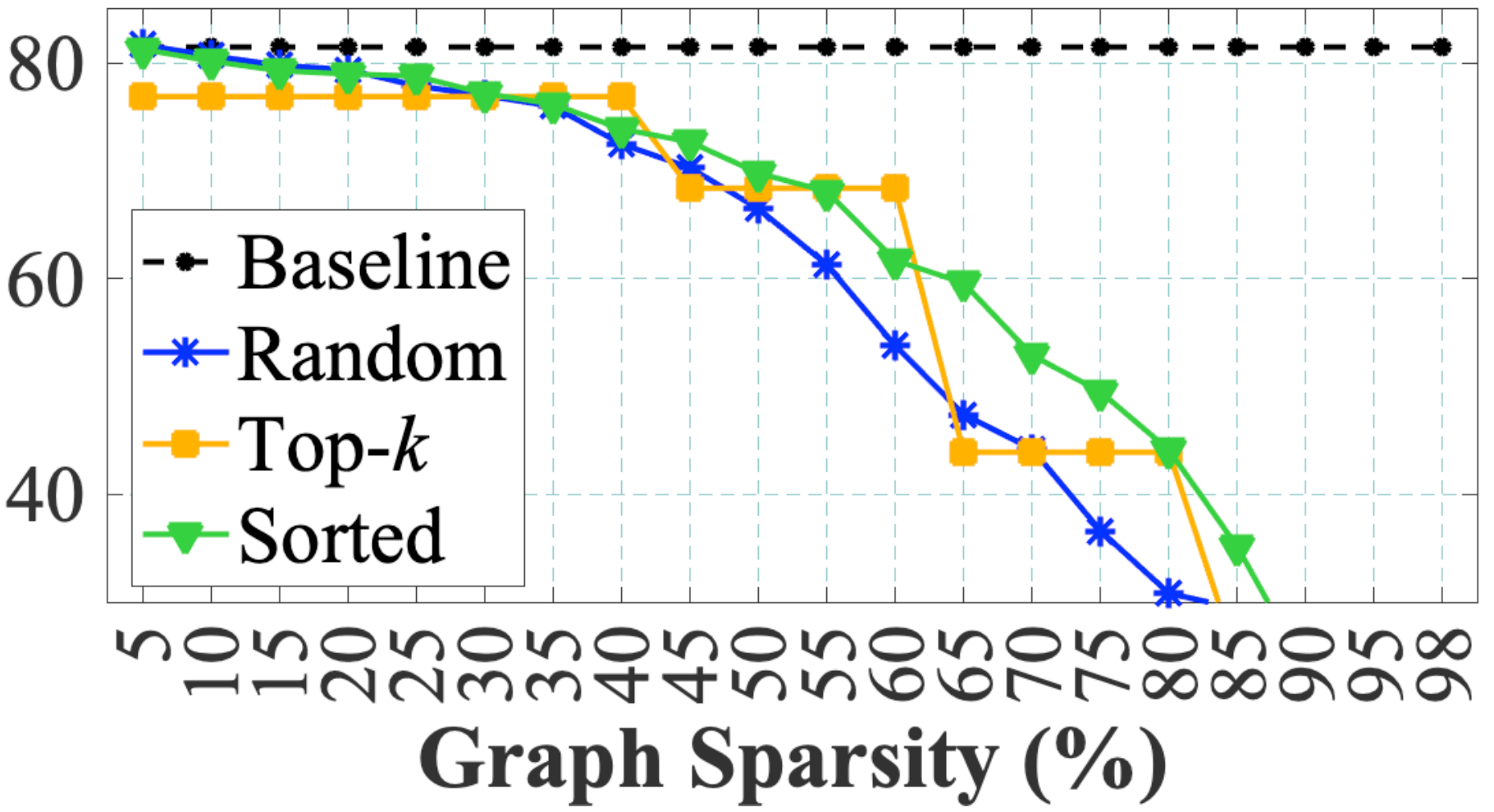}
    \includegraphics[width=0.31\linewidth,height=3.1cm]{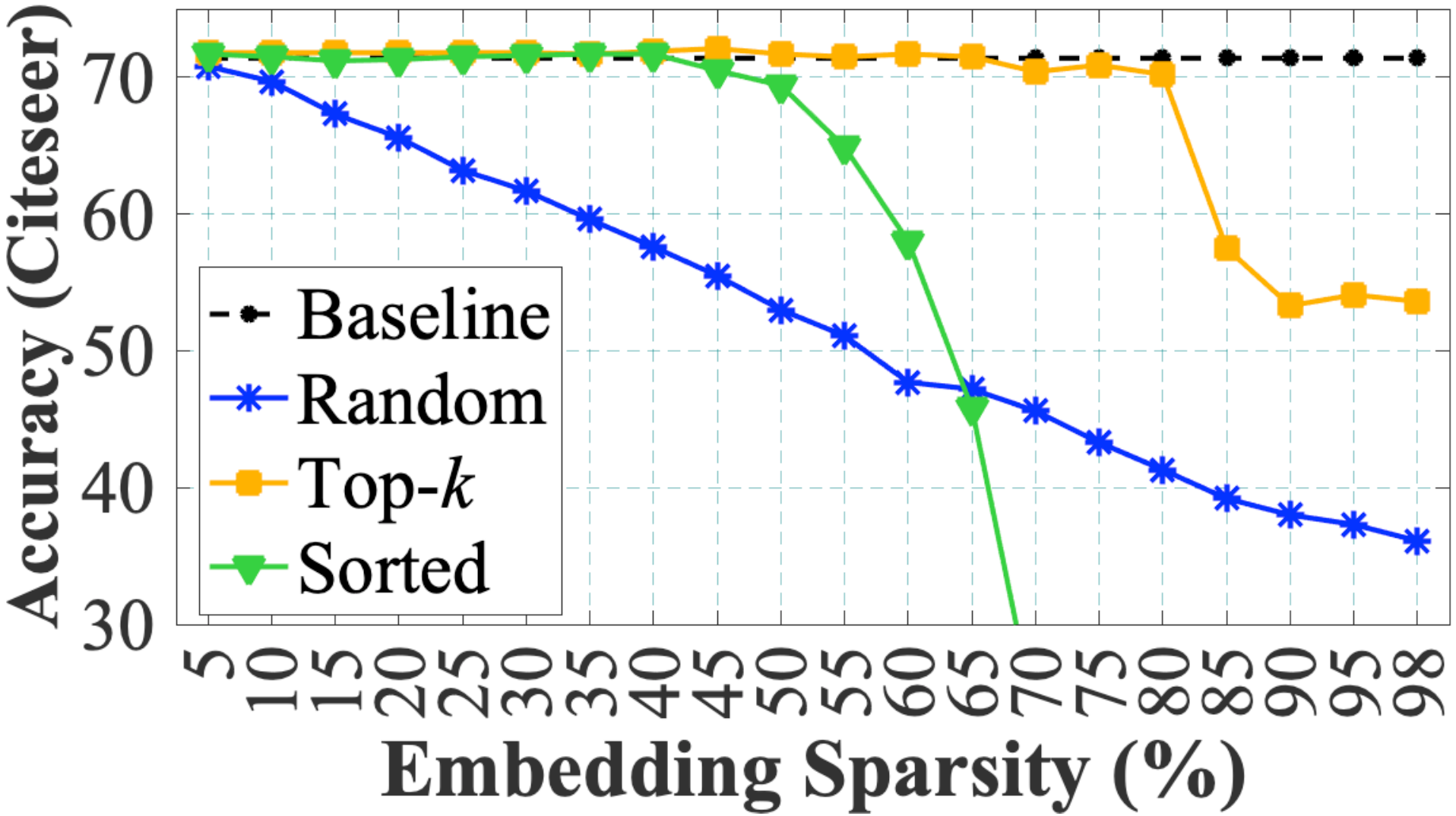}
    \includegraphics[width=0.31\linewidth,height=3.1cm]{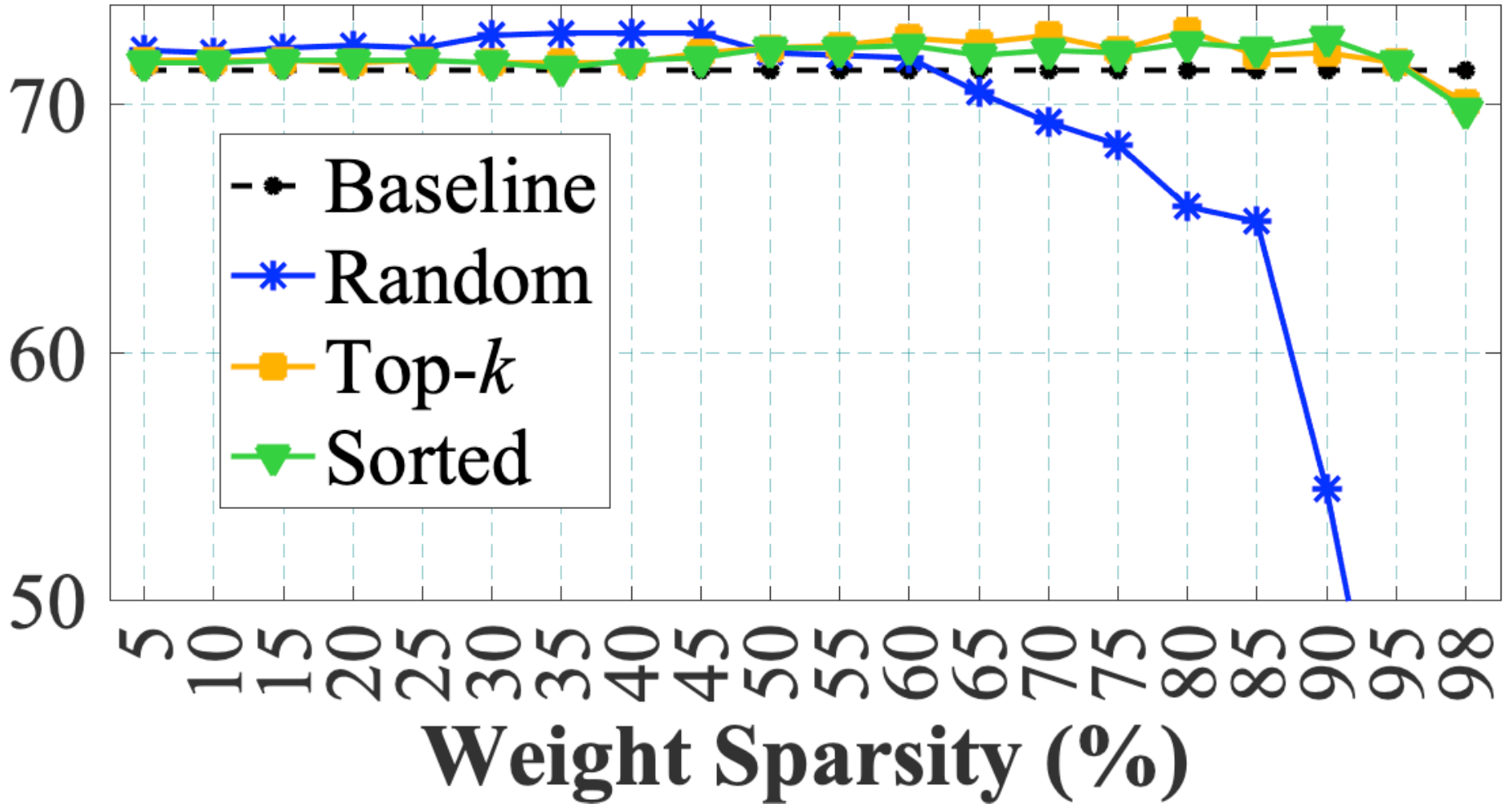}
    \includegraphics[width=0.31\linewidth,height=3.1cm]{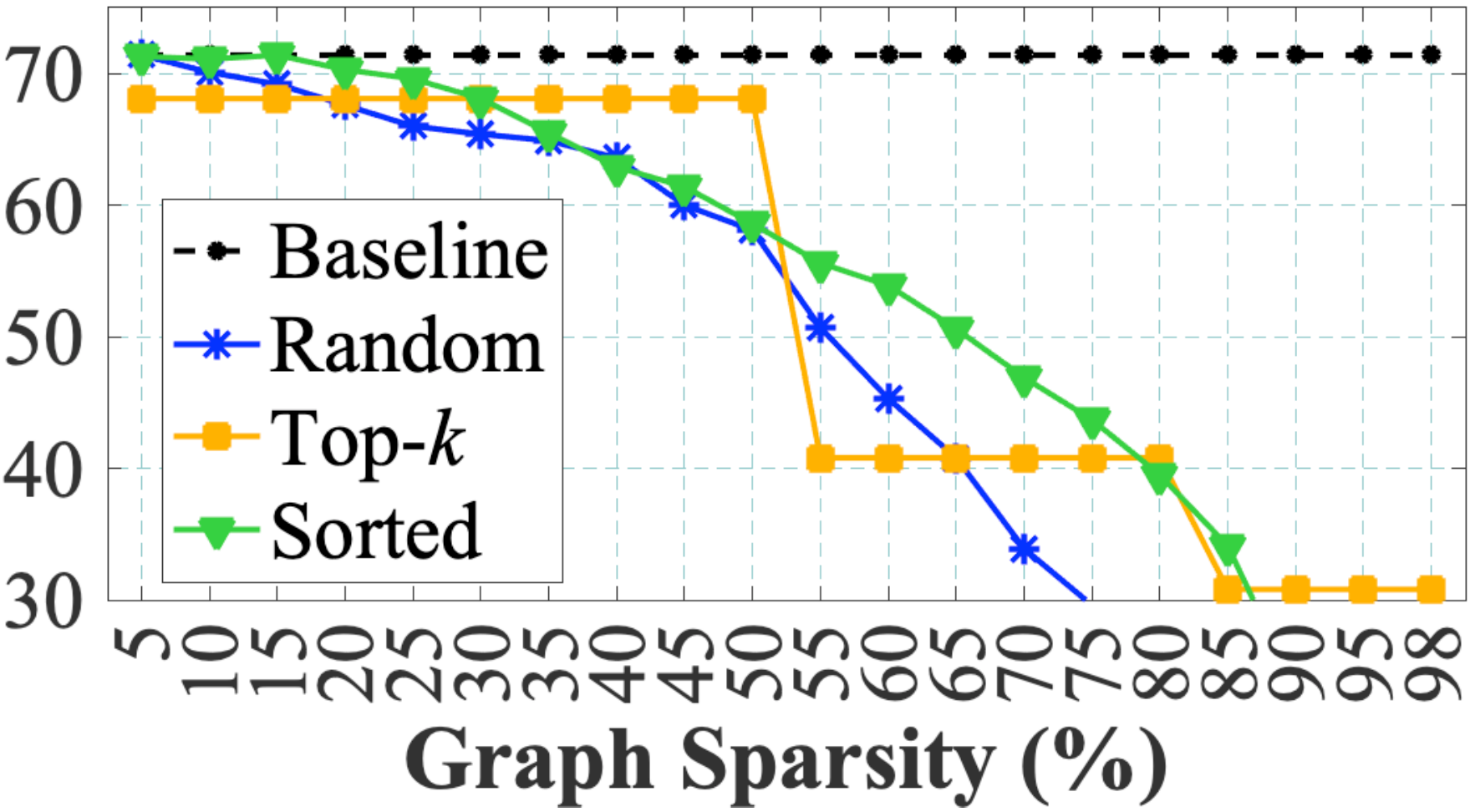}
    \includegraphics[width=0.31\linewidth,height=3.1cm]{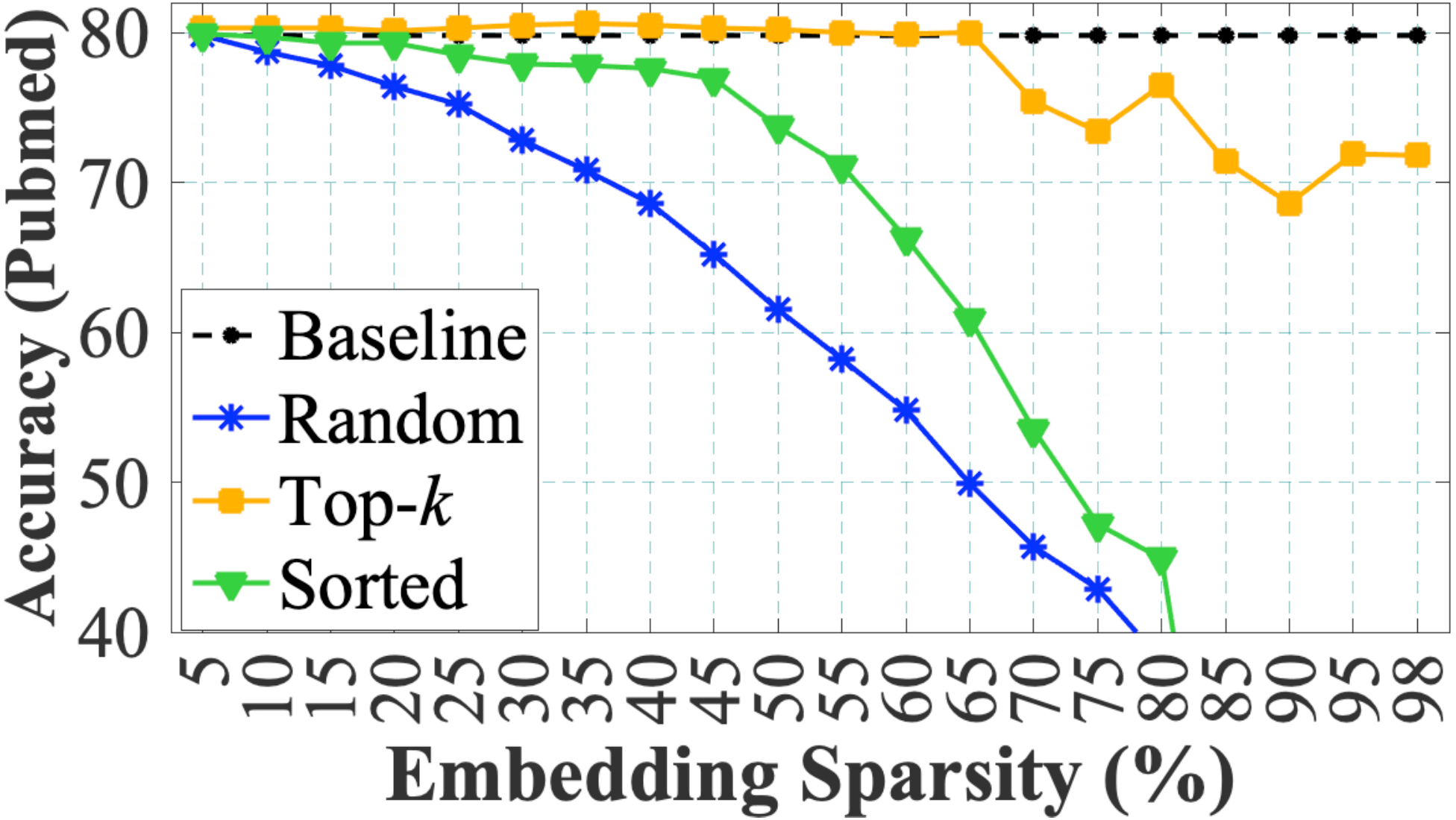}
    \includegraphics[width=0.31\linewidth,height=3.1cm]{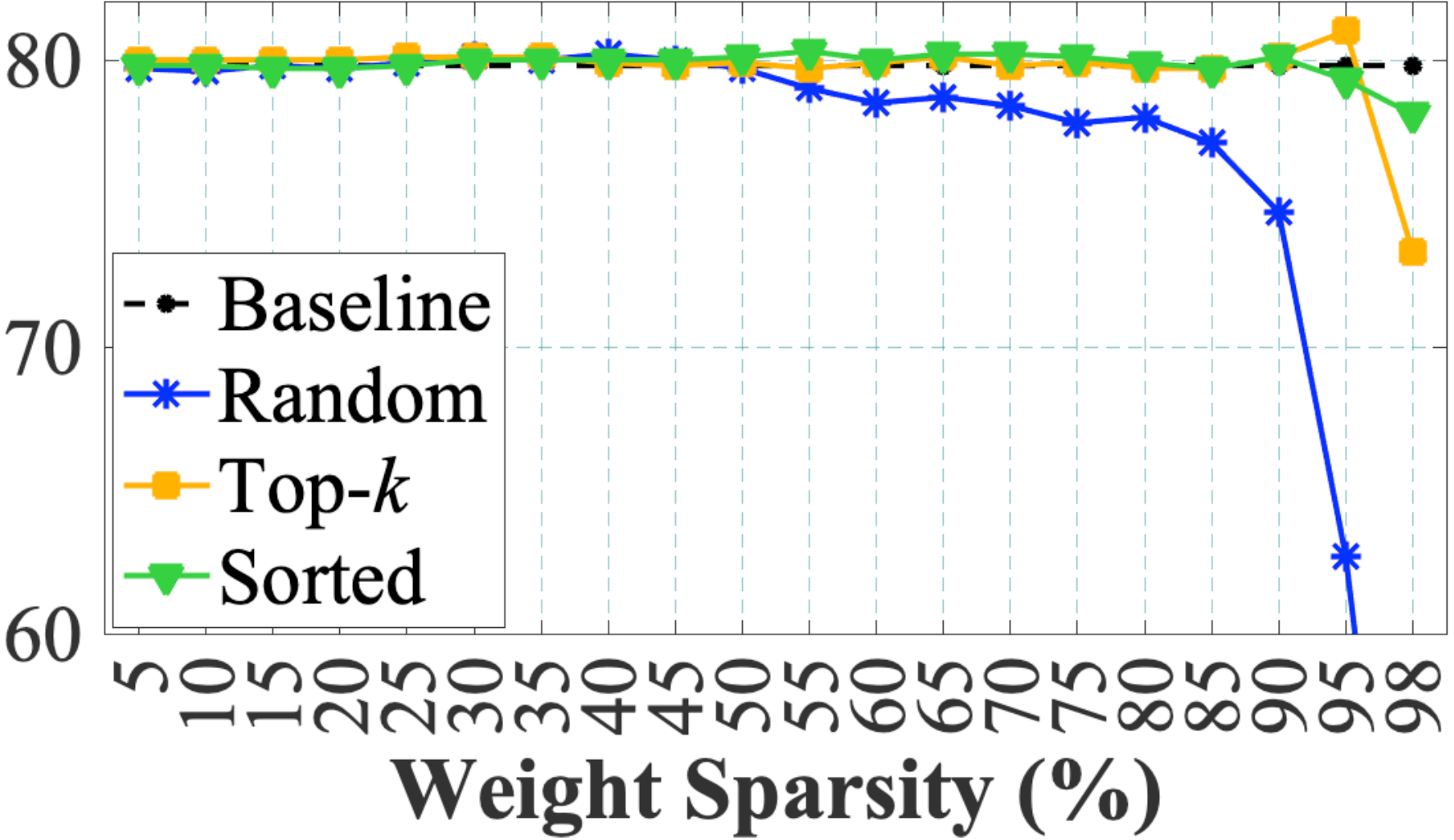}
    \includegraphics[width=0.31\linewidth,height=3.1cm]{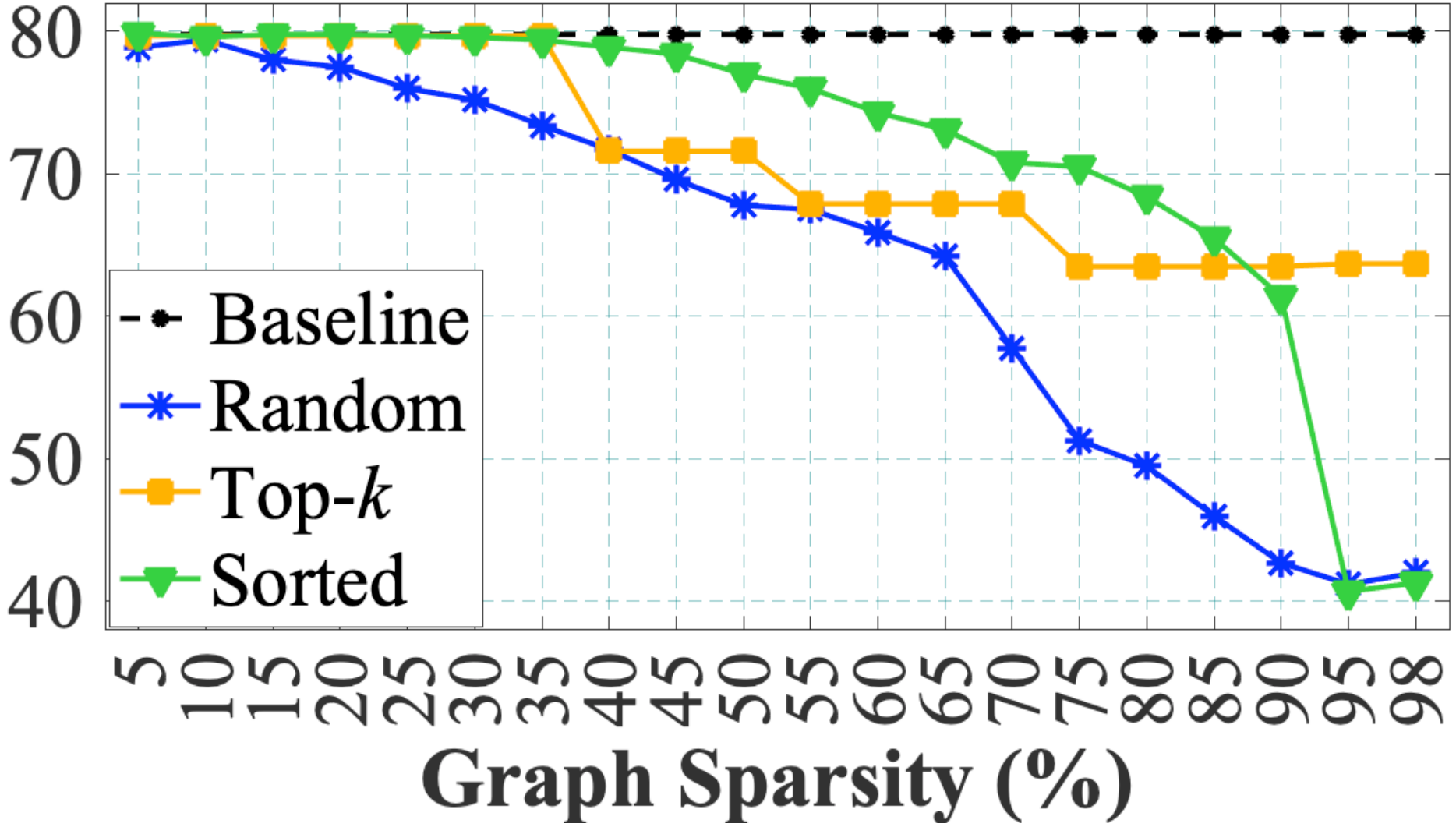}
    %\vspace{-0.45cm}
    \caption{The test accuracy of benchmark graphs using different percentages of sparsity for the node classification tasks. Baseline accuracy is achieved without introducing any sparsity.}
    %\vspace{-0.45cm}
    \label{fig:homophily_sparsity}
\end{figure*}
\textbf{Environment.} We conduct all the experiments of SparseGCN on a server machine. We have implemented the SparseGCN workflow in the PyTorch-Geometric (PyG) framework version 2.0.0~\cite{fey2019pyg}. Unless otherwise mentioned, we use default parameters in a 2-layered GCN model and develop our workflow on top of it. We report the summary of our datasets in Table \ref{tab:dataset} (used by past studies \cite{chen2021unified,you2021gebt}). 
\subsection{Results and Analysis}
\textbf{Individual Sparsity Analysis.} We report the sparsity of different graphs in Fig. \ref{fig:homophily_sparsity}. Here, we show the test accuracy for three different sparsification techniques, namely, (i) Random, (ii) Top-$k$, and (iii) Sorting-based. The $x$-axis shows different percentages of sparsity.

\begin{observation}
\label{obs:embedding}
The Top-$k$ approach can retain the baseline accuracy with a higher percentage of sparsity in the embedding matrix.
\end{observation}
In the sparsity results of embedding matrix (see Fig. ~\ref{fig:homophily_sparsity}, left column), we observe that the Top-$k$ method can retain the baseline accuracy level with a higher percentage of sparsity compared to the random and sorting-based techniques. The Top-$k$ approach is a local sparsification technique that prunes the same number of entries from each row of the embedding matrix based on the absolute value. The impact of the Top-$k$ pruning is similar to reducing the hidden dimension (assuming $k<d$ and $k<f$). %This is similar to reducing the dimensions of the embedding pruning insignificant entries. 
% \checkme{
% Thus, we observe a similar range of accuracy even with 50\% sparsity. We can compare the accuracy of Cora, Citeseer, and Pubmed with respect to the accuracy of 32-dimensional embedding in Fig. \ref{fig:choose_dimension}. We observe that they are similar or differ by a tiny amount.} 
%The random strategy performs the  worst among all methods. The random technique is similar to adding \emph{dropout} to the neural network for regularization, but it does not retain accuracy as seen in Fig. ~\ref{fig:homophily_sparsity}.
%the Top-$k$ approach might serve the purpose better with on-the-fly regularization. this is an unproven comment that we should avoid. 
\begin{observation}
\label{obs:weight}
The Sorting-based approach can retain the baseline accuracy with a higher percentage of sparsity in the weight matrix and adjacency matrix than the other sparsification techniques.
\end{observation}
In the experiments of neural network weight sparsity (see Fig. \ref{fig:homophily_sparsity}, middle column), we add the same percentage of sparsity to both layers of the GCN. We observe that it is similar  to previous studies, which suggests that the sorting-based method can retain the baseline accuracy better than others with a higher percentage of sparsity \cite{chen2021unified}. Notice that the Top-$k$ method can be competitive to the Sorting-based method. The average degree of Cora and Citeseer is comparatively low i.e., they are already very sparse. Making these graphs sparser may disconnect some vertices, which prevents information propagation from neighbors. Thus, even with a small percentage of sparsity in Cora and Citeseer, test accuracy tends to drop. On the other hand, the Pubmed graph is comparatively bigger. Thus, pruning some edges does not affect the test accuracy significantly. For example, the Sorting-based approach can retain the baseline accuracy with up to 40\% graph sparsity, whereas test accuracy drops above 20\% graph sparsity for Cora and Citeseer.  

%From the experimental results of Fig. \ref{fig:homophily_sparsity}, we observe that the homophilic graphs are more sensitive to the graph sparsity than the neural network weight. In other words, the baseline accuracy cannot be preserved by introducing the same maximum possible percentage of sparsity to all matrices in Eqn. \ref{eqn:gcn}.  

% \textbf{Compression ratio of GCN model for Homophilic Graphs.} The weight parameters of a GNN model contribute to a significant part of the training or inference time, especially with the deeper models \cite{li2021training}. Thus, being able to compress or reduce the number of model parameters without sacrificing  the baseline accuracy is desirable to speed up the computations of neural networks \cite{chen2021unified,choudhary2020comprehensive,deng2020model}. In this paper, we analyze the compression ratio of the GCN model empirically employing Sorting-based and Sensitivity-based sparsification techniques. We summarize the results for homophilic graphs in Table \ref{tab:compression} where all techniques retain the baseline test accuracy. For Sensitivity-based pruning, we set $\lambda=1.2$ for Cora and Pubmed graphs, and $\lambda=1.4$ for the Citeseer graph. In Table \ref{tab:compression}, we can see that the Sensitivity-based pruning achieves a better compression ratio (up to 10.6$\times$) for Cora and Citeseer graphs whereas the Sorting-based method achieves a better compression ratio for the Pubmed graph.

\textbf{Comparison with Unified Lotter Ticket Hypothesis.}
We compare our results with ULTH~\cite{chen2021unified} which is a state-of-the-art approach to introduce sparsity to the adjacency matrix and model weights in GNNs. To introduce sparsity using SparseGCN, we use the Top-$k$ approach for the embedding matrix, $\vect{H}$ (Observation \ref{obs:embedding}) and the Sorting-based approach for other matrices. We run both ULTH and SparseGCN for 64-dimensional embedding and report the results in Table \ref{tab:comparison}. We use a grid search technique to find the highest possible sparsity in different matrices retaining a similar level of test accuracy compared to the baseline. We observe that both ULTH and SparseGCN attain a similar level of test accuracy with similar percentages of sparsity in $\vect{\hat{A}}$, and $\vect{W}$; however, SparseGCN introduces additional sparsity to the embedding matrix. More specifically, SparseGCN can introduce 8.4\% to 11.5\% more sparsity to the embedding matrix for these datasets retaining a similar level of test accuracy. For the Pubmed graph, the level of sparsity, as well as test accuracy, outperform ULTH. These empirical results show that a sparsity can be introduced to all matrices of GCN for optimal resource utilization. 
%\vspace{-0.25cm}
\begin{table}[!h]
\centering
\caption{Comparison of combined sparsity and test accuracy between ULTH and SparseGCN. The sparsity in each row of columns $\vect{\hat{A}}$, $\vect{W}$, and $\vect{H}$, are used combinedly.}
%\vspace{-10pt}
\label{tab:comparison}
\arrayrulecolor{black}
\begin{tabular}{c|c|ccc|c} 
\hline
\textbf{Graphs}                    & \textbf{Methods}   & \textbf{$\vect{\hat{A}}$}                                 & \textbf{$\vect{W}$}                                 & \textbf{$\vect{H}$} & \textbf{Accuracy}                                    \\ \hline

\multirow{2}{*}{Cora}     & ULTH~\cite{chen2021unified}      & 14.8\%                                     & 59.1\%                                     & 0.0\%      & 80.0\%                                      \\ 

                          & SparseGCN & 14.5\%                                     & 61.0\%                                     & {\cellcolor[rgb]{0.851,0.918,0.827}}8.4\%      & 80.0\%                                      \\ 
\hline
\multirow{2}{*}{Citeseer} & ULTH~\cite{chen2021unified}       & 19.0\%                                     & 70.1\%                                     & 0.0\%      & 71.5\%                                      \\ 

                          & SparseGCN & 19.1\%                                     & 70.8\%                                     & {\cellcolor[rgb]{0.851,0.918,0.827}}9.8\%      & 71.4\%                                      \\ 
\hline
\multirow{2}{*}{Pubmed}   & ULTH~\cite{chen2021unified}       &  20.6\% & 70.0\% & 0.0\%      &  77.8\% \\ 

                          & SparseGCN & 23.0\%                                     & 75.0\%                                     & {\cellcolor[rgb]{0.851,0.918,0.827}}11.5\%     & 80.0\%                                      \\
\hline
\end{tabular}
\arrayrulecolor{black}
%\vspace{-0.45cm}
\end{table}
%\vspace{-0.65cm}
\section{Conclusions}
In this paper, we study the sparsification of the embedding matrix of any GNN for the first time. We also explore one new sparsification techniques, namely, the Top-$k$ approach. Using the SparseGCN pipeline, we explore the viability of full sparsification of GCN. 
We demonstrate that sparsifying all matrices simultaneously reduces the computational cost more than prior methods that sparsified a subset of matrices. Our theoretical analysis shows that a fully sparse GNN could attain similar baseline performance with fewer computations (additional sparsity reduces MACs in Theorem \ref{theorem:ratio}). In the existing GNN models of popular graph learning frameworks \cite{wang2019dgl,fey2019pyg}, the sparse gradient computation is not explicitly supported. Thus, we aim to implement a fully sparse framework for GNN with efficient sparse-sparse matrix multiplication (SpGEMM) \cite{gu2020bandwidth, srivastava2020matraptor}.
\bibliographystyle{unsrt}
\bibliography{main}

\begin{thebibliography}{10}

\bibitem{wu2020comprehensive}
Zonghan Wu, Shirui Pan, Fengwen Chen, Guodong Long, Chengqi Zhang, and S~Yu
  Philip.
\newblock A comprehensive survey on graph neural networks.
\newblock {\em IEEE transactions on neural networks and learning systems},
  32(1):4--24, 2020.

\bibitem{zhang2020deep}
Ziwei Zhang, Peng Cui, and Wenwu Zhu.
\newblock Deep learning on graphs: A survey.
\newblock {\em IEEE Transactions on Knowledge and Data Engineering}, 2020.

\bibitem{kipf2016semi}
Thomas~N Kipf and Max Welling.
\newblock Semi-supervised classification with graph convolutional networks.
\newblock {\em In Proceedings of International Conference on Learning
  Representations}, 2017.

\bibitem{zhang2018link}
Muhan Zhang and Yixin Chen.
\newblock Link prediction based on graph neural networks.
\newblock {\em Advances in Neural Information Processing Systems},
  31:5165--5175, 2018.

\bibitem{hamilton2017inductive}
William~L Hamilton, Rex Ying, and Jure Leskovec.
\newblock Inductive representation learning on large graphs.
\newblock In {\em Proceedings of the 31st International Conference on Neural
  Information Processing Systems}, pages 1025--1035, 2017.

\bibitem{you2021gebt}
Haoran You, Zhihan Lu, Zijian Zhou, and Yingyan Lin.
\newblock Gebt: Drawing early-bird tickets in graph convolutional network
  training.
\newblock {\em arXiv preprint arXiv:2103.00794}, 2021.

\bibitem{he2016deep}
Kaiming He, Xiangyu Zhang, Shaoqing Ren, and Jian Sun.
\newblock Deep residual learning for image recognition.
\newblock In {\em Proceedings of the IEEE conference on computer vision and
  pattern recognition}, pages 770--778, 2016.

\bibitem{russakovsky2015imagenet}
Olga Russakovsky, Jia Deng, Hao Su, Jonathan Krause, Sanjeev Satheesh, Sean Ma,
  Zhiheng Huang, Andrej Karpathy, Aditya Khosla, Michael Bernstein, et~al.
\newblock Imagenet large scale visual recognition challenge.
\newblock {\em International journal of computer vision}, 115(3):211--252,
  2015.

\bibitem{deng2009imagenet}
Jia Deng, Wei Dong, Richard Socher, Li-Jia Li, Kai Li, and Li~Fei-Fei.
\newblock Imagenet: A large-scale hierarchical image database.
\newblock In {\em 2009 IEEE conference on computer vision and pattern
  recognition}, pages 248--255. Ieee, 2009.

\bibitem{gale2020sparse}
Trevor Gale, Matei Zaharia, Cliff Young, and Erich Elsen.
\newblock Sparse gpu kernels for deep learning.
\newblock In {\em SC20: International Conference for High Performance
  Computing, Networking, Storage and Analysis}, pages 1--14. IEEE, 2020.

\bibitem{chen2021unified}
Tianlong Chen, Yongduo Sui, Xuxi Chen, Aston Zhang, and Zhangyang Wang.
\newblock A unified lottery ticket hypothesis for graph neural networks.
\newblock In {\em International Conference on Machine Learning}, pages
  1695--1706. PMLR, 2021.

\bibitem{li2020sgcn}
Jiayu Li, Tianyun Zhang, Hao Tian, Shengmin Jin, Makan Fardad, and Reza
  Zafarani.
\newblock Sgcn: A graph sparsifier based on graph convolutional networks.
\newblock In {\em Pacific-Asia Conference on Knowledge Discovery and Data
  Mining}, pages 275--287. Springer, 2020.

\bibitem{liu2022graph}
Yixin Liu, Ming Jin, Shirui Pan, Chuan Zhou, Yu~Zheng, Feng Xia, and Philip Yu.
\newblock Graph self-supervised learning: A survey.
\newblock {\em IEEE Transactions on Knowledge and Data Engineering}, 2022.

\bibitem{rahman2021comprehensive}
Md~Khaledur Rahman, Ariful Azad, et~al.
\newblock A comprehensive analytical survey on unsupervised and semi-supervised
  graph representation learning methods.
\newblock {\em arXiv preprint arXiv:2112.10372}, 2021.

\bibitem{rahman2020force2vec}
Md~Khaledur Rahman, Majedul~Haque Sujon, and Ariful Azad.
\newblock Force2vec: Parallel force-directed graph embedding.
\newblock In {\em 2020 IEEE International Conference on Data Mining (ICDM)},
  pages 442--451. IEEE, 2020.

\bibitem{rahman2022markovgnn}
Md~Khaledur Rahman, Abhigya Agrawal, Ariful Azad, et~al.
\newblock Markovgnn: Graph neural networks on markov diffusion.
\newblock {\em arXiv preprint arXiv:2202.02470}, 2022.

\bibitem{velivckovic2017graph}
Petar Veli{\v{c}}kovi{\'c}, Guillem Cucurull, Arantxa Casanova, Adriana Romero,
  Pietro Lio, and Yoshua Bengio.
\newblock Graph attention networks.
\newblock {\em In Proceedings of International Conference on Learning
  Representations}, 2018.

\bibitem{xu2018powerful}
Keyulu Xu, Weihua Hu, Jure Leskovec, and Stefanie Jegelka.
\newblock How powerful are graph neural networks?
\newblock {\em In Proceedings of International Conference on Learning
  Representations}, 2019.

\bibitem{alon2020bottleneck}
Uri Alon and Eran Yahav.
\newblock On the bottleneck of graph neural networks and its practical
  implications.
\newblock {\em In Proceedings of International Conference on Learning
  Representations}, 2021.

\bibitem{oono2019graph}
Kenta Oono and Taiji Suzuki.
\newblock Graph neural networks exponentially lose expressive power for node
  classification.
\newblock {\em In Proceedings of International Conference on Learning
  Representations}, 2020.

\bibitem{chiang2019cluster}
Wei-Lin Chiang, Xuanqing Liu, Si~Si, Yang Li, Samy Bengio, and Cho-Jui Hsieh.
\newblock Cluster-gcn: An efficient algorithm for training deep and large graph
  convolutional networks.
\newblock In {\em Proceedings of the 25th ACM SIGKDD International Conference
  on Knowledge Discovery \& Data Mining}, pages 257--266, 2019.

\bibitem{chen2018fastgcn}
Jie Chen, Tengfei Ma, and Cao Xiao.
\newblock Fastgcn: fast learning with graph convolutional networks via
  importance sampling.
\newblock {\em In Proceedings of International Conference on Learning
  Representations}, 2018.

\bibitem{zeng2019graphsaint}
Hanqing Zeng, Hongkuan Zhou, Ajitesh Srivastava, Rajgopal Kannan, and Viktor
  Prasanna.
\newblock Graphsaint: Graph sampling based inductive learning method.
\newblock {\em In Proceedings of International Conference on Learning
  Representations}, 2020.

\bibitem{ye2021sparse}
Yang Ye and Shihao Ji.
\newblock Sparse graph attention networks.
\newblock {\em IEEE Transactions on Knowledge and Data Engineering}, 2021.

\bibitem{tailor2020degree}
Shyam~A Tailor, Javier Fernandez-Marques, and Nicholas~D Lane.
\newblock Degree-quant: Quantization-aware training for graph neural networks.
\newblock {\em In Proceedings of International Conference on Learning
  Representations}, 2021.

\bibitem{bahri2021binary}
Mehdi Bahri, Ga{\'e}tan Bahl, and Stefanos Zafeiriou.
\newblock Binary graph neural networks.
\newblock In {\em Proceedings of the IEEE/CVF Conference on Computer Vision and
  Pattern Recognition}, pages 9492--9501, 2021.

\bibitem{wang2021bi}
Junfu Wang, Yunhong Wang, Zhen Yang, Liang Yang, and Yuanfang Guo.
\newblock Bi-gcn: Binary graph convolutional network.
\newblock In {\em Proceedings of the IEEE/CVF Conference on Computer Vision and
  Pattern Recognition}, pages 1561--1570, 2021.

\bibitem{koanantakool2016communication}
Penporn Koanantakool, Ariful Azad, Aydin Bulu{\c{c}}, Dmitriy Morozov, Sang-Yun
  Oh, Leonid Oliker, and Katherine Yelick.
\newblock Communication-avoiding parallel sparse-dense matrix-matrix
  multiplication.
\newblock In {\em 2016 IEEE International Parallel and Distributed Processing
  Symposium (IPDPS)}, pages 842--853. IEEE, 2016.

\bibitem{fey2019pyg}
Matthias Fey and Jan~E. Lenssen.
\newblock Fast graph representation learning with {PyTorch Geometric}.
\newblock In {\em ICLR Workshop on Representation Learning on Graphs and
  Manifolds}, 2019.

\bibitem{wang2019dgl}
Minjie Wang, Da~Zheng, Zihao Ye, Quan Gan, Mufei Li, Xiang Song, Jinjing Zhou,
  Chao Ma, Lingfan Yu, Yu~Gai, Tianjun Xiao, Tong He, George Karypis, Jinyang
  Li, and Zheng Zhang.
\newblock Deep graph library: A graph-centric, highly-performant package for
  graph neural networks.
\newblock {\em arXiv preprint arXiv:1909.01315}, 2019.

\bibitem{gu2020bandwidth}
Zhixiang Gu, Jose Moreira, David Edelsohn, and Ariful Azad.
\newblock Bandwidth optimized parallel algorithms for sparse matrix-matrix
  multiplication using propagation blocking.
\newblock In {\em Proceedings of the 32nd ACM Symposium on Parallelism in
  Algorithms and Architectures}, pages 293--303, 2020.

\bibitem{srivastava2020matraptor}
Nitish Srivastava, Hanchen Jin, Jie Liu, David Albonesi, and Zhiru Zhang.
\newblock Matraptor: A sparse-sparse matrix multiplication accelerator based on
  row-wise product.
\newblock In {\em 2020 53rd Annual IEEE/ACM International Symposium on
  Microarchitecture (MICRO)}, pages 766--780. IEEE, 2020.

\end{thebibliography}

\end{document}